\documentclass[12pt]{article}
\usepackage[usenames]{color}
\usepackage{graphicx, subfigure}
\usepackage{amsmath, amsthm, amssymb}
\usepackage{amsfonts}
\usepackage{fullpage}
\usepackage{ifthen}
\usepackage{url}
\usepackage[sort&compress]{natbib}
\usepackage{multirow}
\usepackage[linesnumbered,algoruled,boxed,lined,commentsnumbered]{algorithm2e}
\usepackage{bm}
\usepackage{tikz-qtree}
\usepackage{tikz}
\usetikzlibrary{automata,arrows,positioning,calc}

\makeatletter
\def\@biblabel#1{}
\makeatother

\theoremstyle{plain}
\newtheorem{thm}{Theorem}

\theoremstyle{definition}

\theoremstyle{remark}

\newcommand{\E}{\mathsf{E}}

\newcommand{\prob}{\mathsf{P}}

\newcommand{\pl}{\mathsf{pl}}

\newcommand{\nm}{{\sf N}}
\newcommand{\lnm}{{\sf LogN}}

\newcommand{\gam}{{\sf Gamma}}

\newcommand{\Pareto}{{\sf Pareto}}

\newcommand{\Bern}{{\sf Bern}}

\title{Conformal prediction of future insurance claims in the regression problem}
 
\author{Liang Hong\footnote{Department of Mathematical Sciences, The University of Texas at Dallas, 800 West Campbell Road, Richardson, TX 75080, USA. Tel.:~+1 (972) 883-2161. Email address: liang.hong@utdallas.edu.}}
\date{\today}

\begin{document}

\maketitle

\begin{abstract}
In the current insurance literature,  prediction of insurance claims in the regression problem is often performed with a statistical model.  This model-based approach may potentially suffer from several drawbacks: (i) model misspecification,  (ii) selection effect,  and (iii) lack of finite-sample validity.  This article addresses these three issues simultaneously by employing conformal prediction---a general machine learning strategy for valid predictions.  The proposed method is both model-free and tuning-parameter-free.  It also guarantees finite-sample validity at a pre-assigned coverage probability level. Examples, based on both simulated and real data,  are provided to demonstrate the excellent performance of the proposed method and its applications in insurance, especially regarding meeting the solvency capital requirement of  European insurance regulation, Solvency II.

\smallskip

\end{abstract}

\section{Introduction}

Prediction is one of the most important inferential tasks for actuaries since it forms the basis for  many key aspects of an insurer's business operations, such as premium calculation and reserves estimation.  According to Shmueli (2010),  there are two key goals in data science and statistics: to explain and to predict.  However, these two goals often warrant different approaches.  For example,  as demonstrated in Shmueli (2010),  a wrong model, under some conditions, can even beat the oracle model in prediction, but the same cannot be said for explanation; see,  also,  p.13 of Tukey (1962) for a similar but  broader philosophy.  This paper only concerns prediction.  In the extant insurance literature,  prediction is often performed using either a parametric model or a non-parametric model (e.g., Frees et al. 2014).  In the parametric model approach,  the actuary posits a model,  applies model selection tools to choose the ``best'' model,  trains the chosen model,  and finally makes predictions; see, for example,  Claeskens and Hjort (2008) and Part I of Frees (2010).  While this parametric approach has been widely applied in insurance,  it has several drawbacks. First,  the posited model may be misspecified,  leading to grossly misleading predictions (Hong and Martin 2020).  This prompted some authors to consider robust models (e.g., Brazauskas and Kleefeld 2016).  Second,  scholars recently discovered that model selection might introduce serious biases in prediction; this adverse effect is called the \emph{selection effect} (e.g., Berk et al. 2013; Kuchibhotla et al. 2022); see, for example,  Hong and Martin (2018a, b) for some simple illustrations in insurance.  Though several proposals have been advanced to alleviate the selection effect,  a satisfying answer has yet to be found (Kuchibhotla et al. 2022). Third,  prediction is uncertain and should be probabilistically calibrated to achieve finite-sample validity (Gneiting and Katzfuss 2014; Vovk 2024); unfortunately, no parametric model is known to achieve finite-sample validity.  For these reasons, some actuaries opt for a non-parametric model for predicting insurance claims (e.g., Jeon and Kim 2013; Hong and Martin 2017; Richardson and Hartman 2018).  While most non-parametric models avoid both the model misspecification risk and the selection effect, they can still fall short of our expectations for two reasons.  In the first place,  many of these models contain some tuning parameters, and the choices of these tuning parameters are often debatable.  In general,  there is hardly a choice of the tuning parameter that works across different data sets.  If a data-driven method, such as cross-validation, is employed to select tuning parameters, it amounts to throwing the baby out with the bathwater: prediction will be subject to the selection effect.   Secondly,  predictions made using non-parametric models are usually only asymptotically valid (not finite-sample valid).  Since all real-world applications are done in finite steps,  any asymptotically correct method still leaves much to be desired regarding predictive validity.

Given the above observations,  it is natural to ask whether we can develop a prediction method with the following desiderata: 
\begin{enumerate}
\item[(i)]it is model-free;
\item[(ii)]it is tuning-parameter-free; 
\item[(iii)]it provides finite-sample valid predictions.
\end{enumerate}
The answer is affirmative.  Indeed,  conformal prediction (e.g.,Vovk et al. 2005; Shafter and Vovk 2008) is a general machine learning strategy that can be used to create such a predictive method.  Hong and Martin (2021) introduce conformal prediction to insurance for the first time.  Hong (2023) and Chen et al. (2024) apply conformal prediction to credibility estimation and cat bond pricing, respectively.  While Hong and Martin (2021) and Hong (2023) develop predictive methods with all the above desiderata for unsupervised learning, their methods, if adapted directly to supervised learning, would ignore the information supplied by predictors in the regression problem.  This is certainly not desirable.  Chen et al. (2024) make important contributions to the cat bond literature by applying conformal prediction to predict cat bond prices.  But their method does not guarantee finite-sample validity of predictions. The reason is that their method cannot be implemented for every possible value of the response value; see Section~\ref{sec:proposed} for a discussion on why finite-sample validity is not automatic when an actuary applies conformal prediction.   Therefore,  a predictive method, with the above three desiderata, for a regression problem is yet to be developed. The main goal of this paper is to propose such a method.

The remainder of the paper proceeds as follows.  Section~2 reviews the model-based approach and the model-free approach in the regression setup and argues why a model-free approach is desirable in prediction.  Section~\ref{sec:conf} provides readers with the necessary background by reviewing conformal prediction.  Then,  Section~3 describes the proposed method in detail.   Next, Section~4 uses simulation to show the excellent performance of the proposed method.  It also gives a real-data example to demonstrate possible applications of the proposed method in insurance.  Finally,  Section~5 concludes the paper with concluding remarks. 

\section{Model-based approaches versus model-free approaches for prediction}

In supervised learning,  the true data-generating mechanism can be described as follows:
\begin{equation}
\label{eq:true}
Y=g^\star(X)+\epsilon,
\end{equation}
where $Y$ is the response variable,  $X$ is the vector of $p$ features/predictors, $g^\star$ is an unknown deterministic function,  $\epsilon$ is a random error term such that $\epsilon\sim F^\star$, and $F^\star$ is an unknown distribution function for $Y-g^\star(X)$.  Throughout this article,  $p$ is a positive integer, and upper-case letters will denote random variables/vectors while the corresponding lower-case letters stand for their realizations.  Suppose we have observed independent and identically distributed (iid) pairs $Z_i = (X_i, Y_i)$, for $i=1,\ldots,n$ from (\ref{eq:true}), where $n$ is the sample size.  Our chief goal is to predict the next response $Y_{n+1}$ at a randomly sampled feature $X_{n+1}$, based on past observations of $Z^n=  \{Z_1,\ldots,Z_n\}$.   Since we aim at finite-sample valid predictions (see (\ref{eq:jointvalidity}) in Section 3) and no point predictor can achieve this, we will only consider interval prediction. In general, there are two different types of approaches for performing prediction: a model-based approach and a model-free approach.  

In a \emph{model-based approach},   an actuary first posits a \emph{statistical learning model}:
\[
\mathcal{M}=\{(\mathcal{G},  \mathcal{F})\},
\]
where $\mathcal{G}$ is a family of deterministic functions and $\mathcal{F}$ is a family of distribution functions.  For example, in the linear regression model, we take $\mathcal{G}$ to be $\{g(x_1, \ldots, x_p)=a_1x_1+\ldots+a_px_p\mid (a_1, \ldots, a_p)\in\mathbb{R}^p\}$ and $\mathcal{F}$ to be $\{\nm(0, \sigma^2)| \sigma^2 \in (0,\infty)\}$, where $\nm(\mu, \sigma^2)$ denotes the normal distribution function with mean $\mu$ and variance $\sigma^2$.  Next, the actuary uses data to obtain estimators $\widehat{g^\star}$ of $g^\star$ and $\widehat{F^\star}$ of $F^\star$ such that $\widehat{g^\star}\in \mathcal{G}$ and $\widehat{F^\star}\in \mathcal{F}$.  Finally, the actuary 
use $\widehat{g^\star}$ and $\widehat{F^\star}$  to create a prediction interval
\[
(L(Z^n), U(Z^n)),
\]
where the lower bound $L(Z^n)$ can be $-\infty$ and the upper bound $U(Z^n)$ can be $+\infty$.  
If $\mathcal{G}$ and $\mathcal{F}$ both can be characterized by finitely many parameters,   the model $\mathcal{M}$ is said to be a \emph{parametric} model.  For example, a linear model is a parametric model.  If either $\mathcal{G}$ or $\mathcal{F}$ cannot be characterized by finitely many parameters, we say $\mathcal{M}$ is a \emph{non-parametric} model.  For example,  if we take $\{g(x_1, \ldots, x_p)=a_1x_1+\ldots+a_px_p\mid (a_1, \ldots, a_p)\in\mathbb{R}^p\}$ and $\mathcal{F}=\{F\mid \text{$F$ is a distribution function supported on $\mathbb{R}$}\}$, then $\mathcal{M}$ will be a non-parametric model.  

Two challenges come into consideration when we employ a model-based approach for prediction: one is model misspecification, and the other is the effect of selection.  If $g^\star\in\mathcal{G}$ and $F^\star\in\mathcal{F}$, we say the model $\mathcal{M}$ is \emph{well-specified} or \emph{correct}; otherwise, we say $\mathcal{M}$ is \emph{misspecified} or \emph{wrong}.   In some cases,  it can be unquestionable that a model is wrong.  For example, if $Y_i$'s are observed insurance claims, then a linear model is patently misspecified because the response variable in a linear model can take negative values.  However,  we can never be certain that a model is correct because neither $g^\star$ nor $F^\star$ is observable. Therefore, even if a model seems to be reasonable,  there is still a danger of model misspecification.  Hong and Martin (2020) demonstrate that model misspecification can lead to grossly misleading predictions.  One way to mitigate model misspecification is to apply some robust techniques in the spirit of Brazaukas and Kleefeld (2016).  Another option is to use a non-parametric model.  Either way,  there is another danger lurking behind the scenes.  When actuaries apply a parametric model,  they often employ a model selection tool, such as AIC or BIC,  to choose the ``best'' model. In such a case,  the model is essentially treated as random (since the model selection process is data-dependent), but the existing formulas for predictions usually assume the model is fixed (i.e., conditional on the fact that the model has been chosen).  This discrepancy can give rise to serious biases in prediction---a phenomenon called the \emph{selection effect};  see  Hong et al. (2018b) for an illustration.  Note that if cross-validation is used to choose tuning parameters in a non-parametric model (e.g.,  bandwidth in a kernel density model), the selection effect might still wreak havoc on predictions.  Some non-parametric models, such as the Dirichlet mixture of lognormal model in Hong and Martin (2017) may circumvent the selection effect.  How to overcome the selection effect in a parametric model is an important topic of research in data science and statistics.  At the time of writing, no efficient method has been found; see Kuchibhotla et al. (2022) for a review of recent results on the effect of selection.  

The above two challenges in a model-based approach stem from the fact that prediction is conducted via a model. Therefore, one obvious way to circumvent these two difficulties is to perform prediction in one stroke: prediction is made from data directly without specifying a model.  Such an approach is called a \emph{model-free approach} for prediction.  
Technically,  a model-free method for prediction is a mapping that maps data $Z^n$ directly to a prediction interval $(L(Z^n), U(Z^n))$ without training a model.  Figure~1  illustrates the difference between a model-based approach and a model-free approach.  There are two paths from ``Data'' to ``Prediction''.  One first goes from ``Data'' to ``Model'' and then from ``Model'' to ``Prediction''; this path denotes the model-based approach.  The other path goes from ``Data'' directly to ``Prediction''; it stands for the model-free approach.

\begin{center}
	\begin{tikzpicture}[->, >=stealth', auto, semithick, node distance=3cm]
	\tikzstyle{squarednode}=[fill=white,draw=black,thick,text=black,scale=2]
	\node[squarednode]    (A)                     {Data};     
	\node[squarednode]    (B)[above right of=A]   {Model};
	\node[squarednode]     (C)[below right of=B] {Prediction};
	\path
	(A) edge[left,above]			node{}      (C)
	      edge[left,above]	                  node{}	(B)
	(B) edge[right, below]                 node{}       (C);
	\end{tikzpicture}          
\end{center}
\begin{center}
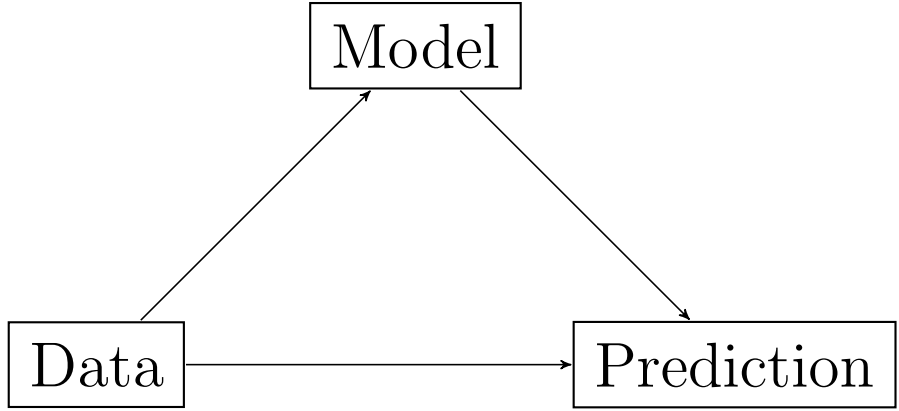
	Figure~1: Model-based approach versus model-free approach for prediction. 
\end{center}

A distinguishing mark of a model-free approach is that the notion of a model never enters the scene.  Many machine learning techniques, such as decision trees and $K$-nearest neighbors (e.g.,  Trevor et al.  2009),  share this feature.  Some elementary estimators in non-parametric statistics, such as sample mean and sample variance,  have the same feature too,  though they are designed for estimation, not prediction.  A model-free approach for prediction can be considered a prediction method in non-parametric statistics.  However,  as mentioned above,  a non-parametric statistical model might still be subject to the adverse impact of the selection effect.  Moreover,  non-parametric predictive models (e.g.,  Jeon and Kim 2013, Hong and Martin 2017, and Richardson and Hartman 2018) can ensure asymptotically valid prediction at best, and none of them guarantees finite-sample validity.  Finally,  non-parametric prediction methods exist in the literature (e.g., Wilks 1941; Fligner and Wolfe 1976; Frey 2013, Hong and Martin 2021, Hong 2023).  But these methods are all designed for unsupervised learning.  Adopting them directly for prediction in the regression problem means we need to disregard the information provided by $X_1, \ldots, X_n$,  which is undesirable.  Therefore, we need to seek a new method in supervised learning that is (i) model-free, (ii) tuning-parameter-free, and (iii) guarantees finite-sample validity.

We conclude this section with a few words on finite-sample validity and asymptotic validity of a prediction region.  Suppose $\prob$ is the distribution of $Z_1=(X_1, Y_1)$, and  $D_{\alpha}(X_{n+1}, Z^n)$ is a $100(1-\alpha)\%$ prediction region (not necessarily constructed using conformal prediction) based on observations $Z^n$ and $X_{n+1}$.  One of the key criteria for assessing the quality of $D_{\alpha}(X_{n+1}, Z^n)$ is its validity.   We say $D_{\alpha}(X_{n+1}, Z^n)$ is \emph{asymptotically valid} if
\begin{equation}\label{eq:asympticvalidity}
 \lim_{n\rightarrow \infty}  \prob^{n+1}\{Y_{n+1}\in D_{\alpha}(X_{n+1}; Z^n) \}\geq 1-\alpha\quad \text{for all $\prob$},
\end{equation}
where $\prob^{n+1}$ is the joint distribution for $(X_1, Y_1), \ldots, (X_n, Y_n), (X_{n+1}, Y_{n+1})$.  Informally, asymptotic validity means if the actuary could sample from the true data-generating mechanism (\ref{eq:true}) forever,  then eventually his coverage probability of his prediction region would be no less than the advertised confidence level $1-\alpha$, regardless of the distribution $\prob$ of $(X_1, Y_1)$.  However,  the actuary can only have a finite sample in practice, and asymptotic validity does not tell us the sample size $n$ needed to ensure that the coverage probability of the prediction region $D_{\alpha}(X_{n+1}, Z^n)$ is no less than $1-\alpha$.  For this reason,  it would be practically more desirable if the coverage probability of $D_{\alpha}(X_{n+1}, Z^n)$ is at least $1-\alpha$ for any sample size $n$ and any distribution $\prob$,  that is, 
\begin{equation}\label{eq:finite-samplevalidity}
 \prob^{n+1}\{Y_{n+1}\in D_{\alpha}(X_{n+1}; Z^n) \}\geq 1-\alpha\quad \text{for all $(n, \prob)$},
 \end{equation}
where $\prob^{n+1}$ is the joint distribution for $(X_1, Y_1), \ldots, (X_n, Y_n), (X_{n+1}, Y_{n+1})$.  If a prediction region  $D_{\alpha}(X_{n+1}, Z^n)$ satisfies (\ref{eq:finite-samplevalidity}), we say it is \emph{finite-sample validity}.  In this case, given any sample of size $n$ and a confidence level $1-\alpha$, the actuary can construct a $100(1-\alpha)\%$ prediction region $D_{\alpha}(X_{n+1}, Z^n)$ satisfying (\ref{eq:finite-samplevalidity}),  irrespective of what distribution $\prob$ is.

\section{Conformal prediction}
\label{sec:conf}

Conformal prediction is a general learning strategy designed for generating prediction regions at a pre-determined coverage probability level.  Shafer and Vovk (2008) give a user-friendly introduction to conformal prediction.  For a comprehensive treatment, see Vovk et al. (2005).  Conformal prediction can be used to tackle both supervised learning and unsupervised learning problems.  Here we only review conformal prediction in supervised learning.

Conformal prediction can be applied when data is exchangeable.  For any $n\geq 1$,  we say the finite sequence $Z_1, Z_2, \ldots, Z_n$ is \emph{exchangeable,} if the joint distribution of $(Z_1, \ldots, Z_n)$ is the same as the joint distribution of $(Z_{\sigma(1)}, \ldots, Z_{\sigma(n)})$, where $\sigma: \{1, 2, \ldots, n\}\rightarrow \{1, 2, \ldots, n\}$ is any permutation on the set $ \{1, 2, \ldots, n\}$.  Note that the exchangeability assumption is weaker than the widely used iid (independent identically distributed) assumption.  In particular,  if $Z_1, \ldots, Z_n$ are iid,  they are exchangeable (e.g., Chapter 1,  Schervish 1995).  

Our goal is to predict the next response $Y_{n+1}$ at a randomly sampled feature $X_{n+1}$, based on past observations of $Z^n=  \{Z_1,\ldots,Z_n\}$.   To apply conformal prediction,  the actuary starts with a real-valued deterministic  mapping $M$ of two arguments $(B, z)\mapsto M(B, z)$, where the first argument $B$ is a \emph{bag}, i.e., a collection, of observed data and the second argument $z$ is a provisional value of a future observation to be predicted based on the data in $B$.  $M$ is called a \emph{non-conformity measure}. It measures the degree of non-conformity of the provisional value $z$ with the data in $B$.  When the provisional value $z$ agrees with the data in $B$, $M(B, z)$ will be relatively small; otherwise, it will tend to be  large.  For example, if $z$ is real-valued and $B=\{z_1, \ldots, z_n\}$, then we can take $M(B, z)=|\hat{m}_B(x)-y|$, where $\hat m_B(\cdot)$ is an estimate of the conditional mean function, $\E(Y \mid X=x)$, based on the bag $B$.   There is no unique choice of the non-conformity measure.  In practice,  the non-conformity is chosen at the discretion of the actuary.  Once a non-conformity measure is chosen, the actuary runs the conformal prediction algorithm---Algorithm~\ref{algo:conformal.s} to predict the value of the next label $Y_{n+1}$ at a randomly sampled feature $X_{n+1}=x_{n+1}$.  \\

\begin{algorithm}[H]
Initialize: data $z^n = \{z_1,\ldots,z_n\}$ and $x_{n+1}$, non-conformity measure $M$, and a possible $y$ value\;
Set $z_{n+1} = (x_{n+1}, y)$ and write $z^{n+1} = z^n \cup \{z_{n+1} \}$\;
Define $\mu_i = M(z^{n+1} \setminus \{z_i\}, z_i)$ for $i=1,\ldots,n,n+1$\;
Compute $\pl_{z^n}(x_{n+1}, y) = (n+1)^{-1} \sum_{i=1}^{n+1} 1\{\mu_i \geq \mu_{n+1}\}$\;
Return $\pl_{z^n}(x_{n+1}, y)$\;
\caption{\bf Conformal prediction (supervised learning)}
\label{algo:conformal.s}
\end{algorithm}

\bigskip

In Algorithm~\ref{algo:conformal.s}, $1_E$ denotes the indicator function of an event $E$.  That is, 
\[
1_E(x)=\left\{ \begin{array}{ll}
		                          1,& \hbox{if  $x\in E$,} \\
		                          0,& \hbox{otherwise.} \\
		                          \end{array}
		                         \right.
\]
The quantity $\mu_i$,  called the $i$-th \emph{non-conformity score},  assigns a numerical value to $z_i$ to indicate how much $z_i$ agrees with the data in the $i$-th augmented bag $\widetilde{B}_i=z^n\cup\{z_{n+1}\}\backslash\{z_i\}$, where $z_i$ itself is excluded to avoid biases as in leave-one-out cross-validation.  Algorithm~\ref{algo:conformal.s} corresponds to the function $\pl_{z^n}$, that is, for a given $x_{n+1}$ and a possible $y$, it outputs $\pl_{z^n}(x_{n+1}, y)$. The function $\pl_{z^n}$ is called  the \emph{plausibility function}.  It outputs a value between $0$ and $1$, based on all non-conformity scores.  The output of $\pl_{z^n}$ shows how plausible $z$ is a value of $Z_{n+1}$, based on the available data $Z^n=z^n$ and $X_{n+1}=x_{n+1}$.  Using the output of the plausibility function $\pl_{z^n}$,  the actuary can construct a $100(1-\alpha)\%$ conformal prediction region
\begin{equation}
\label{eq:region}
C_\alpha(x; Z^n) = \{y: \pl_{Z^n}(x, y) > \alpha\},
\end{equation}
where $0<\alpha<1$.  The next theorem establishes the finite-sample validity of the conformal prediction region in (\ref{eq:region}).

\begin{thm}
\label{thm:valid}
If $\prob$ denotes the distribution  (i.e., the law) of an exchangeable sequence $Z_1,Z_2,\ldots$, then write $\prob^{n+1}$ for the corresponding joint distribution of $Z^{n+1}=\{Z_1,\ldots,Z_n,Z_{n+1}\}$.  For $0<\alpha<1$, define $t_n(\alpha) = (n+1)^{-1}\lfloor (n+1)\alpha \rfloor$, where $\lfloor a \rfloor$ denotes the greatest integer less than or equal to $a$.  Then
\begin{equation}
\label{eq:valid}
\sup \prob^{n+1}\{ \pl_{Z^n}(Z_{n+1}) \leq t_n(\alpha) \} \leq \alpha \quad \text{for all $n$ and all $\alpha \in (0,1)$},
\end{equation}
where the supremum is over all distributions $\prob$ for the exchangeable sequence.
\end{thm}

\begin{proof}
The proof is similar to its counterpart in the unsupervised learning context (e.g., Hong and Martin 2021).  In great detail, $\mu_1,\ldots,\mu_n, \mu_{n+1}$ are exchangeable because $Z_1,Z_2,\ldots$ are exchangeable. Hence, the rank of $\mu_{n+1}$ follows the discrete uniform distribution on the set $\{1,\ldots,n,n+1\}$.  By its definition, the plausibility function $\pl_{Z^n}(Z_{n+1})$ is proportional to the rank of $\mu_{n+1}$. Therefore, $\pl_{Z^n}(Z_{n+1})$ also follows the discrete uniform distribution on the set $\{1/(n+1), 2/(n+1), \ldots, 1\}$. For a given $0<\alpha<1$, if $(n+1)\alpha$ is an integer, then
$\prob^{n+1}\{\pl_{n+1}(Z_{n+1})\leq t_n(\alpha) \}=\alpha$.  Otherwise, we will have
$\prob^{n+1}\{ \pl_{n+1}(Z_{n+1})\leq  t_n(\alpha)  \} \leq t_n(\alpha)< \alpha$.  Therefore,   (\ref{eq:valid}) holds. 
\end{proof}

Theorem~\ref{thm:valid} implies that the prediction region in (\ref{eq:region}) is \emph{finite-sample (joint) valid}:
\begin{equation}\label{eq:jointvalidity}
\prob^{n+1}\{Y_{n+1}\in C_{\alpha}(X_{n+1}; Z^n) \}\geq 1-\alpha\quad \text{for all $(n, \prob)$},
\end{equation}
where $\prob^{n+1}$ is the joint distribution for $(X_1, Y_1), \ldots, (X_n, Y_n), (X_{n+1}, Y_{n+1})$.  Namely, the coverage probability of using the $100(1-\alpha)\%$ conformal prediction region $C_{\alpha}(x, Z^n)$ for prediction is no less than $1-\alpha$ for all sample size $n$ and all distribution $\prob$.  Note that this notion of finite-sample validity of the prediction region is distinct from a more desirable property of \emph{conditional validity}: 
\[ \prob^{n+1}\{Y_{n+1}\in C_{\alpha}(X_{n+1}; Z^n)\mid X_{n+1}=x \}\geq 1-\alpha\quad \text{for all $(n,\prob)$ and almost all $x$}. \]
By the iterated expectation formula, we know conditional validity implies joint validity:
\begin{equation*}
\label{eq:vadlityrelationship1}
\prob^{n+1}\{Y_{n+1}\in C_{\alpha}(X_{n+1}; Z^n) \}=\E\left[ \prob^{n+1}\{Y_{n+1}\in C_{\alpha}(X_{n+1}; Z^n)\mid X_{n+1}\} \right],
\end{equation*}
where the expectation is taken with respect to the distribution of $X_{n+1}$.   Informally, the conditional validity states that the probability of a correct prediction is no less than $1-\alpha$ each time the actuary makes a prediction.  In contrast, the joint validity says that the overall rate of correct predictions is at least $1-\alpha$.  While conditional validity is more desirable than joint validity, it is practically unachievable.  Indeed, Vovk (2012) and Lei and Wasserman (2014) show that it is impossible to achieve conditional validity property with a bounded prediction region $C_{\alpha}(X_{n+1}; Z^n)$ in supervised learning.  Hong (2023) establishes a similar result for unsupervised learning.  Therefore,  the finite-sample joint validity is practically the best we can hope for.

\section{Proposed method for insurance applications}
\label{sec:proposed}

In many insurance applications,  $Y$ denotes the insurance claims or something alike, and $X$, a $p$-dimensional vector in $\mathbb{R}^p$,  symbolizes $p$ features.  In this paper, we only consider the regression problem, i.e.,  the case where the response variable $Y$ is continuous; however, see,  Section~\ref{sec:concluding} for some remarks for the case of a categorical dependent variable.  Therefore, we assume that $Y$ is a non-negative random variable supported on the non-negative real line $\mathbb{R}_+= [0, \infty)$.  An insurer often needs to set its risk capital to an adequate level to comply with insurance regulations (e.g., European Solvency II 2009, article 101) or implement its risk management plan.  Therefore,  an insurer needs to generate a finite-sample valid prediction interval of the form $[0,  b)$, where the upper bound $b=b(X_{n+1}; Z^n)>0$ is a possibly data-dependent quantity.  There is a paucity of methods for achieving this in the insurance literature.  In particular, if an actuary applies a generalized linear model, then predictions will be subject to model misspecification risk. Even if the model assumption can be fully justified,  the selection effect can still adversely affect the prediction accuracy (e.g., Hong et al. 2018a, b).  While existing non-parametric regression methods (e.g., those methods based on kernel smoothing, basis expansion,  regression splines, etc) might alleviate these issues,   they,   however,  can only generate asymptotically valid prediction intervals for $Y_{n+1}$, let alone the challenge of choosing the tuning parameters of these models.   Fortunately,  conformal prediction can come to the rescue. 

It is clear from Algorithm~\ref{algo:conformal.s} and Theorem~\ref{thm:valid} that conformal prediction is a model-free method.  In particular,  finite-sample validity,  given by (\ref{eq:jointvalidity}),  holds regardless of the non-conformity measure an actuary chooses.  However,  to apply conformal prediction, an actuary must overcome several challenges. First,  the shape of the conformal prediction region $C_{\alpha}(x, Z^n)$  depends on the choice of the non-conformity measure.  In particular,  when the response variable $Y\in\mathbb{R}_+$, there is no guarantee that the conformal prediction region $C_{\alpha}(x, Z^n)$ is an interval of the form $[0, b)$.   The shape of $C_{\alpha}(x, Z^n)$ can be a disjoint union of several intervals; see, for instance,  Section~2.3 of Vovk et al. (2005) and Lei et al. (2013).  If this happens,  the resulting conformal region will not be practically useful.  Second,  perhaps more importantly,  a close scrutiny of the argument in Section~\ref{sec:conf} reveals that determination of the $100\%$ conformal prediction region $C_{\alpha}(x; Z^n)$ requires one to run Algorithm~\ref{algo:conformal.s} for all possible $y$ value of $Y$.  However,  in practice an actuary can only run  Algorithm~\ref{algo:conformal.s} for  finitely many $y$ values.  This means $C_{\alpha}(x, Z^n)$ cannot be determined exactly in general,  compromising finite-sample validity---one of the key selling points of conformal prediction.   Hong and Martin (2021) and Hong (2023) both take some special non-conformity measures and prove that corresponding $C_{\alpha}(x, Z^n)$ is equivalent to another prediction interval which can be exactly determined (i.e., determined in finitely many steps).  This is why the conformal prediction intervals in these two papers are finite-sample valid.   In general, finite-sample validity is only possible but not automatically guaranteed when the actuary applies conformal prediction.  Finally,  even for an approximate determination of the conformal prediction region $C_{\alpha}(x; Z^n)$, i.e.,   we only consider a grid of $y$ values, the computation required could be prohibitively expensive.  

To address these challenges together,   we consider the following non-conformity measure.  Let $B=\{z_1, \ldots, z_n\}$ be a bag of  observations of size $n$. For $i=1, \ldots, n$, let $z_i=(x_{i1}, \ldots, x_{ip}, y_i)$ be the $i$-th observations in $B$.  That is, $x_{ij}$ denotes the $i$-th observation of the the $j$-th feature.  Suppose $z=(x_1, \ldots, x_p, y)$ is a provisional value of a future observation to be predicted.  We define $M$ to be 
\begin{equation}
\label{eq:newnonconform}
M(B, z)=y-\left[\frac{1}{n}\sum_{j=1}^p x_{j}  + \sum_{i=1}^n\left(y_i-\frac{1}{n}\sum_{j=1}^p x_{ij} \right)\right].
\end{equation}

In (\ref{eq:newnonconform}),  $\sum_{i=1}^n(y_i-\sum_{j=1}^p x_{ij}/n)$ is the only term that depends on the bag $B$.  Informally,  it measures the aggregate difference between the response $Y$ and the sum of the features per data point.  However,  it is worth underlining that we are not proposing a model like $Y=X_1/n+\ldots +X_p/n+\varepsilon$, where $\varepsilon$ is some random error term.  Conformal prediction is a model-free method; hence, there is no model to consider.  $M$ in (\ref{eq:newnonconform}) is just a non-conformity measure we propose to address the above challenges.   This non-conformity measure possesses several desirable properties: (i) the resulting prediction region is an interval of form $[0, b)$,  (ii) there is a closed-form formula for the resulting conformal prediction interval,   making implementation simple,   (iii) determination of the resulting conformal prediction interval can be done in finitely many steps.  Properties~(i)---(ii) immediately follow from the following theorem.  In addition, Theorem~\ref{thm:nonconform} shows the $100(1-\alpha)\%$ conformal conformal prediction region in this case is equivalent to an interval based on some order statistics,  which can be determined in finitely many steps. Hence, there is no need to calculate the plausibility function $\pl_{Z^n}(X_{n+1}, y)$ for all possible $y$.

\begin{thm}
\label{thm:nonconform}
Suppose the non-conformity measure is given by (\ref{eq:newnonconform}). For $1\leq i\leq n$, let $W_i=Y_i+\sum_{i=1}^p (X_{(n+1)j}-X_{ij})/n$. Then for any $0<\alpha<1$,  the $100(1-\alpha)\%$ conformal prediction region $C_{\alpha}(X_{n+1}; Z^n)$ for $Y_{n+1}$ equals $(0, W_{(k)})$, where $k=\min\{n, \lfloor (n+1)(1-\alpha)+1\rfloor \}$.
\end{thm}

\begin{proof}
For notational simplicity, we let $S_i=\sum_{j=1}^p X_{ij}$ for $1\leq i\leq n$.  We have
\begin{eqnarray*}
\mu_i &=& M(Z^{n+1}\backslash Z_i, Z_i)=Y_i-\left[S_i/n+\sum_{k=1, k\neq i}^{n+1}(Y_k-S_k/n) \right]\\
	&=& Y_i-\left[S_i/n+\sum_{k=1}^{n}(Y_k-S_k/n)+(Y_{n+1}-S_{n+1}/n)-(Y_i-S_i/n)\right]\\
	&=& 2(Y_i-S_i/n)-(Y_{n+1}-S_{n+1}/n)-\sum_{k=1}^n(Y_k-S_k/n),\text{ for $1\leq i\leq n$},
\end{eqnarray*}
and 
\begin{eqnarray*}
\mu_{n+1} &=& M(Z^{n}, Z_{n+1})=Y_{n+1}-\left[S_{n+1}/n+\sum_{k=1}^{n}(Y_k-S_k/n) \right]\\
	&=& (Y_{n+1}-S_{n+1}/n)-\sum_{k=1}^n(Y_k-S_k/n).
\end{eqnarray*}
Therefore, $\mu_i\geq \mu_{n+1}$ if and only if 
\[
2(Y_i-S_i/n)-(Y_{n+1}-S_{n+1}/n) \geq (Y_{n+1}-S_{n+1}/n),
\]
which amounts to saying $\mu_i\geq \mu_{n+1}$ if and only  if $Y_{n+1} \leq Y_i+(S_{n+1}-S_i)/n=W_i$ for all $1\leq i\leq n$. Therefore, $Y_{n+1}\leq W_{(k)}$ implies that $Y_{n+1}$ is less than at least $(n+1)-k+1$ $W_i$'s.  Now the theorem follows from the definitions of $\pl_{z^n}(x, y)$ and $C_{\alpha}(x; Z^n)$.
\end{proof}

\section{Examples}

\subsection{Simulated data}
Here we use simulation to demonstrate the excellent performance of the proposed conformal prediction interval. 
Throughout this section,  the following notation will be used:
\begin{enumerate}
\item[$\bullet$]$\Bern(p)$ denotes the Bernoulli distribution with success probability $0<p<1$;
\item[$\bullet$]$\lnm(\mu, \sigma)$  stands for  the lognormal distribution with parameters $\mu$ and $\sigma$, i.e,  the distribution of $e^{\mu+\sigma S}$, where $S$ follows the standard normal distribution;
  \item [$\bullet$]$\gam(a, b)$ represents the gamma distribution with shape parameter $a$ and rate parameter $b$,  or equivalently, the gamma distribution whose probability density function is 
              \[p(x)=\frac{b^a}{\Gamma(a)}x^{a-1}e^{-b x}, \quad x>0,\]
               where $\Gamma(\alpha)=\int_0^\infty t^{\alpha-1} e^{-t} dt$ is the Gamma function;
  \item [$\bullet$]$\Pareto(\eta, \beta)$ symbolizes the Type II Pareto distribution with the probability density function
              \[h(x)=\frac{\eta \beta^{\eta}}{(x+\beta)^{\eta+1}}, \quad x>0.\]
\end{enumerate}
We only consider prediction intervals of the form $[0, b)$ where $b>0$ is possibly data-dependent.  Let $Y$ stands for the claim amount in a line of business. Then a prediction interval of this form,  if finite-sample valid at the confidence level $99.5\%$, is what the insurer needs to meet the solvency capital requirement of Solvency II.

\subsection*{Example 1}
Suppose the true data-generating mechanism is 
\begin{equation}
\label{eq:oracle1}
Y=X_1+\epsilon,
\end{equation}
where $X_1\sim \gam(2, 2.5)$,  $\varepsilon\sim \gam(0.04, 2.5)$, and $X_1$ and $\varepsilon$ are independent.  The closure property of the gamma distribution implies $Y\sim \gam(2.04, 2.5)$.  The oracle data-generating mechanism here is not a linear regression model because the error term $\epsilon$ in (\ref{eq:oracle1}) is not normally distributed.  The rationale for such a choice is that $Y$ is non-negative in many insurance applications. Since $\gam(2.04, 2.5)$ is a member of the exponential family, the oracle data-generating mechanism is a generalized linear model.  Note that here the oracle $99.5\%$ prediction interval has a closed-form: $[0,  b)$ where $b$ is the $99.5\%$ percentile of $\gam(2.04, 2.5)$. Also,  the oracle $99.5\%$ prediction interval is the shortest one in terms of the mean length of a prediction interval.

We compare the proposed $99.5\%$ conformal prediction interval with the $99.5\%$ prediction interval based on two other methods: (a)~generalized linear models and (b)~random forests.  For $n=200$, we generate $N=2,000$  random samples of size $n+1$ from each of these three methods.  For each sample,  we use the values of the response variable in the first $n $ sample points and all the values of the features to construct the $100(1-\alpha)\%$ conformal prediction interval. Then we use the value of the response variable of the last sample point, i.e,  the value of $(n+1)$-st response variable,  to test its validity.  We estimate the coverage probability of each conformal prediction interval as $K/N$, where $K$ is the number of times it contains the value of the $(n+1)$-st response variable.  For the generalized linear model, we take the distribution for the response $Y$ to be the gamma distribution and the link function to be the identity function.  Hence,  the generalized linear model we consider here is a gamma regression model. Note that this particular generalized linear model is well-specified.  $\texttt{glm()}$ function in the R language is used to fit this gamma regression model.  However, a gamma regression model does not admit a closed-form formula for prediction intervals. Therefore, we apply the parametric bootstrapping method in Olive et al. (2022) to determine the  $99.5\%$ prediction interval.  Throughout the simulation, the bootstrapping sample size is chosen to be $B=500$.  The random forest method is implemented with the R package $\texttt{randomForest}$.  The number of trees is fixed at $500$ in the simulation. As for the $99.5\%$ random forest prediction interval, no closed-formula is available, either.  But,  a large-sample approximation is well-known: 
\[
\left( \min\{0, |\hat{g}(X_{n+1})-\hat{q}_{99.5\%}|\},  \hat{g}(X_{n+1})+\hat{q}_{99.5\%} \right),
\]
where $\hat{g}(X_{n+1})$ is the prediction of $Y_{n+1}$, given  $X_{n+1}$, made by the fitted random forest,  and $\hat{q}_{99.5\%}$ is the $99.5\%$ empirical quantile of $|Y_1-\hat{g}(X_1)|, \ldots, |Y_n-\hat{g}(X_n)|$.  Coverage probabilities of the GLM and random forest prediction intervals are calculated in the same way as we do for the conformal prediction interval.  Table~\ref{table:ex1} summarizes the coverage probabilities and the mean length relative to the oracle for all three methods.

\begin{table}[!ht]
\begin{center}
\scalebox{0.9}{
\begin{tabular}{l|cccccc} 
\hline
Method     		&   Coverage Probability &  Mean Length Relative to the Oracle Length&\\
 \hline
Conformal     		& 99.50\%&      1.0750\\
GLM         			 & 99.75\% & 1.1053     \\
 Random Forest         & 98.10\% &      0.3065\\
\hline

\hline
\end{tabular}
}
\end{center}
\caption{\small {Coverage probabilities and mean length relative to oracle length for the $99.5\%$ prediction intervals in Example~1,  based on conformal prediction, gamma regression, and random forest.}}
\label{table:ex1}
\end{table}

Conformal prediction intervals and GLM prediction intervals both achieve the nominal coverage level, while random forest prediction intervals fall short of providing adequate coverage.  The mean length relative to the oracle length indicates the efficiency of each method.  We need to exercise caution when interpreting this quantity.  For example,  if each of the $N=2,000$ prediction intervals in the simulation were the same as the oracle prediction interval, then the mean length to the oracle length would be exactly $1$.  However, if each upper bound of the $N=2,000$ prediction intervals in the simulation were very close to the new response value $Y_{n+1}$ in that iteration,  then the prediction method under consideration could be excellent for point prediction,  but,  the mean length relative to the oracle length would be considerably less than 1 since most response values will be much less the $99.5\%$ percentile of $\gam(2.04, 2.5)$---the upper bound of the oracle prediction interval.  Therefore,  the low mean length relative to the oracle length of random forest prediction intervals does not mean random forest prediction intervals perform terribly,  it simply mean they are ``too aggressive'' in terms of the mean length: they tend to be too close to the oracle prediction interval,  with each of them being either slightly under-cover or slightly over-cover,  but, on average,  they under-cover and fail to achieve the desired coverage level.  Conformal prediction intervals and GLM prediction intervals are both efficient, though they have higher mean lengths relative to the oracle length than random forest prediction intervals.  There is no clear winner between conformal prediction intervals and GLM prediction intervals.  However, the gamma regression model is well-specified in this case.  In the next example, we will see that GLM prediction intervals will not perform as well as conformal prediction intervals when the model is misspecified.

\subsection*{Example 2}
We perform the same simulation as in Example~1, except that now we assume the true data-generating mechanism is
\begin{equation}
\label{eq:oracle2}
Y=X_1+\epsilon,
\end{equation}
where $X_1\sim \Pareto(3, 10)$,  $\varepsilon\sim \gam(0.04, 2.5)$, and $X_1$ and $\varepsilon$ are independent.  All parameters are chosen to be the same as in Example~1: $\alpha=0.5\%$, $N=2,000$, $n=200$, $B=500$,  and the number of trees in random forest is $500$.  In this case, the oracle $99.5\%$ prediction interval is unknown, but we can approximate it as $[0,\hat{b})$, where $\hat{b}$ is the $99.5\%$ empirical quantile.  Table~\ref{table:ex2} summarizes the simulation results.

\begin{table}[!ht]
\begin{center}
\scalebox{0.9}{
\begin{tabular}{l|cccccc} 
\hline
Method     		&   Coverage Probability &  Mean Length Relative to the Oracle Length&\\
 \hline
Conformal     		& 99.60\%&      1.1335\\
GLM         			 & 57.80\% &    0.0699  \\
 Random Forest          & 98.80\% &     0.3105 \\
\hline

\hline
\end{tabular}
}
\end{center}
\caption{\small {Coverage probabilities and mean length relative to oracle length for the $99.5\%$ prediction intervals in Example~2, based on conformal prediction, gamma regression, and random forest.}}
\label{table:ex2}
\end{table}

In this case, the exact distribution of $Y$ is unknown. Since $X_1$ has a heavy-tailed distribution and plays a dominant role in (\ref{eq:oracle2}),  the distribution of $Y$ is heavy-tailed too.  Therefore,  a gamma regression model with the identity link function is a misspecified model.   As a result,  GLM prediction intervals seriously under-cover.   Table~\ref{table:summaryex2} displays some summary statistics of the $2,000$ GLM prediction interval lengths.  Note that $\E[Y]>\E[X_1]=5$, which is more than $20\%$ larger than the $90\%$ empirical percentile of these GLM prediction interval lengths.  Therefore,  most of them are too short to contain the new response variable $Y_{n+1}$. This explains why the GLM mean length to the oracle length is so low.

\begin{table}[!ht]
\begin{center}
\scalebox{0.9}{
\begin{tabular}{l|ccccccc}
\hline
 Variable       & Minimum & 1st Quartile & Median & 3rd Quartile & $90\%$ percentile & Maximum & Mean\\
PI Length    & 2.4260 & 3.0370 & 3.2800 & 3.6868 & 3.9025 & 834.953 & 3.8678\\
\hline
\end{tabular}
}
\end{center}
\caption{\small  Some summary statistics of the $2,000$ GLM prediction interval lengths in Example~2.  }
\label{table:summaryex2}
\end{table}

The performance of random forest prediction intervals is similar to what we saw in Example~1. This is no surprise, because random forest is also a model-free approach and not subject to model misspecification.  As for efficiency, random forest prediction intervals are more conservative than GLM prediction intervals. However,  they are still ``too aggressive'', under-estimating the degree of uncertainty in the data.  Conformal prediction is the only method here that achieves the nominal coverage probability level.  It is also efficient with a mean length relative to the oracle length of $1.1335$.

\subsection*{Example 3}
Many insurance regression problems involve some categorical predictors. These categorical predictors are often coded as numerical scores or treated as dummy variables (e.g.,  Section~4.1,  Frees 2010).  Here we consider one such case.  Suppose the oracle data-generating mechanism is 
\begin{equation}
\label{eq:oracle3}
Y=X_1^2+3 X_2 +  2X_3 +\epsilon,
\end{equation}
where $\varepsilon$ is a random error term, $X_1\sim \Pareto(1.5, 4)$, $X_2\sim \Bern(1/3)$,  $X_3\sim \lnm(1, 0.5)$, and $X_1, X_2, X_3$, and $\varepsilon\sim \gam(2, 4)$ are independent.  We perform the same simulation with the same parameter values as in Example~1. But here the GLM model is a three-predictor (not one-predictor) gamma regression model with the identity link function.  Table~\ref{table:ex3} summarizes the results.

\begin{table}[!ht]
\begin{center}
\scalebox{0.9}{
\begin{tabular}{l|cccccc} 
\hline
Method     		&   Coverage Probability &  Mean Length Relative to the Oracle Length&\\
 \hline
Conformal     		& 99.65\%&      1.7379\\
GLM         			 & 70.20\% &    0.2780\\
 Random Forest          & 98.80\% &      0.5469\\
\hline

\hline
\end{tabular}
}
\end{center}
\caption{\small {Coverage probabilities and mean length relative to oracle length for the $99.5\%$ prediction intervals in Example~3, based on conformal prediction, gamma regression, and random forest.}}
\label{table:ex3}
\end{table}

Similar to what we observed in Example~2, conformal prediction is the only method that achieves the nominal coverage probability level.  However,  this time conformal prediction intervals are more conservative.  This is due to the fact that the data here has a higher degree of uncertainty than before.  Note that $\E[X_1]$, $\E[X_2]$, and $\E[\epsilon]$ are $8$, $3.08$, and $0.5$, respectively. The second term in (\ref{eq:oracle3}), being either $3$ or $0$, adds a nontrivial layer of uncertainty to the data.  Therefore,  the conservativeness of the conformal prediction intervals is the price we must pay for the finite-sample validity.  Indeed, all three methods behave more conservatively than in Example~2.  The performance of 
random forest prediction intervals is similar to what we saw in the previous two examples: they tend to be aggressive in length, but, on average,  they under-cover.  The gamma regression model with the identity link is still a misspecified model.  Model misspecification again wreaks havoc on prediction: the resulting GLM prediction intervals are too short to achieve the nominal coverage probability level.

\subsection{Personal injury insurance claims data}

In this real-data example, we consider the personal injury insurance claims data from Jong and Heller (2008).  This data set contains information about $22,036$ settled personal injury claims. These claims resulted from accidents that occurred from July 1989 and January 1999.  Where these claims occurred is unknown.  Claims settled with no payment are not included.
There are $11$ variables in total.  Table~\ref{table:predictors} lists the names, meanings, and values of all these variables.  \\

\begin{table}[!ht]
\begin{center}
\scalebox{0.9}{
\begin{tabular}{lll}
\hline
Variable     & Denotation & Range\\
\hline
total      & settled amount & $\$10–\$4, 490, 000$\\
inj1        &     injury 1 &  1 (no injury),  2, 3, 4, 5, 6 (fatal), 9 (not recorded)\\
inj2        &     injury 2 &  1 (no injury),  2, 3, 4, 5, 6 (fatal), 9 (not recorded)\\
inj3        &     injury 3 &  1 (no injury),  2, 3, 4, 5, 6 (fatal), 9 (not recorded)\\
inj4        &     injury 4 &  1 (no injury),  2, 3, 4, 5, 6 (fatal), 9 (not recorded)\\
inj5        &     injury 5 &  1 (no injury),  2, 3, 4, 5, 6 (fatal), 9 (not recorded)\\
legrep   &  legal representation &  0  (no) , 1  (yes) \\
accmonth    &    accident month              & 1 (July 1989), ..., 120  (June 1999)\\
repmonth    &    reporting month              & 1 (July 1989), ..., 120  (June 1999)\\
finmonth    &    finalization month              & 1 (July 1989), ..., 120  (June 1999)\\
op$\_$time    &     operational time      & 0.1--99.1\\
\hline
\end{tabular}
}
\end{center}
\caption{\small  Summary statistics of the personal injury insurance data. }
\label{table:predictors}
\end{table}

Except for the variable ``total'', all other variables are categorical and have been assigned numerical scores.  For variables ``inj1'',  ``inj2'',  ``inj3'', ``inj4'', and ``inj5'', there are lots of missing values. We assign $0$ to each of these missing values. Table~\ref{table:summary_real} displays summary statistics of all these variables. \\

\begin{table}[!ht]
\begin{center}
\scalebox{0.9}{
\begin{tabular}{l|ccccccc}
\hline
 Variable       & Minimum & 1st Quartile & Median & Mean & 3rd Quartile & Maximum \\
 \hline
total     &  10 & 6,279  &  13,845 & 38,367 &  35,123 & 4,485, 797 \\
inj1    &  1.00 &  1.00 &  1.00 & 1.83 &  2.00 & 9.00  \\
inj2    &  0.00 &   0.00 & 1.00&  0.67 & 1.00 &   9.00\\
inj3    &  0.00 &  0.00 &  0.00 & 0.42 & 1.00  &9.00 \\
inj4    &  0.00 &  0.00 &  0.00 & 0.24 & 1.00  &9.00 \\
inj5    &  0.00 &  0.00 &  0.00 & 0.16 & 1.00  &9.00 \\
legrep    &  0.00 & 0.00 & 1.00 & 0.64 & 1.00   & 1.00\\
accmonth   & 1.00 & 44.00  & 64.00 & 62.06& 82.00  & 115.00 \\
repmonth    &  15.00 & 54.00 & 69.00 & 71.33 & 85.00  & 116.00 \\
finmonth    &  49.00 & 72.00 & 91.00 & 88.37 & 106.00   & 117.00\\
op$\_$time  & 0.10  & 23.00  & 45.90  & 46.33  & 69.30  &  99.10\\
\hline
\end{tabular}
}
\end{center}
\caption{\small  Summary statistics of the personal injury insurance data.  }
\label{table:summary_real}
\end{table}

We apply the proposed method to this data set with `total' being the response variable and all other variables being predictors.  Table~\ref{table:ex4} shows the upper bound of the $100(1-\alpha)\%$ conformal prediction interval in Theorem~\ref{thm:nonconform} for three $\alpha$ levels: $ 0.90, 0.95$, and $0.995$.  Although we do not know the true distribution here,  we know from Section~3 that a conformal prediction interval is provably valid; hence  the three conformal prediction intervals in Table~\ref{table:summary_real} all achieve the nominal coverage probability level for their respective $\alpha$ value.  In particular,  if the insurer wants to comply with Solvency II, they can set their risk capital level to $544,093$ for this line of business.

\begin{table}[!ht]
\begin{center}
\scalebox{0.9}{
\begin{tabular}{l|ccc}
         $1-\alpha$ & prediction interval upper bound\\
 \hline
             0.90      & 85,333 \\
              0.95     & 148, 993\\
	   0.995     & 544, 093
\end{tabular}
}
\end{center}
\caption{\small The upper bounds of the $100(1-\alpha)\%$ conformal prediction interval of the form $(0, b)$ where $b>0$, based on personal injury  insurance data for various values of $\alpha$.  }
\label{table:ex4}
\end{table}


%
\section{Concluding remarks}
\label{sec:concluding}
Explanation and prediction are the two major tasks in data science and statistics. When it comes to actuarial science, prediction is usually the main goal.   As Example~2 demonstrates, if a prediction interval is created using a parametric model (e.g.,  a generalized linear model),  it can undercover when the proposed model fails to match the true data-generating mechanism.   This model misspecification issue is avoided when a prediction interval is based on a non-parametric method (e.g., Random Forests). However, such a prediction interval may not be finite-sample valid. For example,  Examples~1 to 3 show that Random Forests prediction intervals are not finite-sample valid.   In practice,  a finite-sample valid prediction interval is crucial for an insurer to understand the risk associated with their business and to set risk capital at an appropriate level.  In particular,  the insurer can apply it to comply with the risk capital requirements in Solvency II, which dictates that the insurer must be able to cover losses with a $99.5\%$.  This article proposed a new non-conformity measure to create a new prediction interval for the response variable in the regression problem.  This prediction interval has multiple merits: (i) it avoids both model misspecification risk and the selection effect, (ii) it has a closed-form formula and hence is easy to compute, and (iii) it guarantees finite-sample validity.  Our simulation study shows that this prediction interval achieves finite-sample validity at the confidence level $99.5\%$ across different scenarios.  In this paper, we have restricted the domain of the response variable $Y$ to $\mathbb{R}_+$ since the response variable is non-negative in most insurance applications.  However, it is easy to see that the results in this paper extend to the case where $Y\in\mathbb{R}$.   Finally,  the method proposed in this paper does not apply to the case of a categorical response variable.  To apply conformal prediction to predict a categorical response variable would require a different strategy than the one proposed in this article.

\section*{Acknowledgments}
I thank the Editor for his support and an anonymous reviewer for many helpful comments and suggestions.

\section*{References}

\begin{description}

\item{} Berk, R. Brown, L., Buja, A. Zhuang, K. and Zhao, L.~(2013). Valid post-selection inference, \emph{Annals of Statistics}, 41, 802--837.

\item{} Brazauskas, Y.~and Kleefeld, A.~(2016). Modeling severity and measuring tail risk of Norwegian fire claims. \emph{North American Actuarial Journal} 20(1), 1--16.

\item{} Chen,  X., Li, H.,  Lu,  Y.~and Zhou, R.~(2024).  Pricing catastrophe bonds---a probabilistic machine learning approach.  \url{https://papers.ssrn.com/sol3/papers.cfm?abstract_id=4809670}

\item{} Claeskens, G.~and Hjort, N.L.~(2008).  \emph{Model Selection and Model Averaging}.  Cambridge University Press: Cambridge, UK. 

\item{} de Jong,  P.~and Heller, G.Z.~(2008).  \emph{Generalized Linear Models for Insurance Data}. Cambridge University Press: Cambridge, UK. 

\item Fligner, M.~and Wolfe, D.A.~(1976). Some applications of sample analogues to the probability integral transformation and a coverage property. The American Statistician~30(2), 78--85.

\item{} Frees, E. W.~(2010). \emph{Regression Modeling with Actuarial and Financial Applications}, Cambridge: Cambridge University Press.

\item{} Frees, E.~W., Derrig, R.~A.~and Meyers, G.~(2014). \emph{Predictive Modeling Applications in Actuarial Science, Vol. I: Predictive Modeling Techniques}, Cambridge: Cambridge University Press.

\item{} Frey, J.~(2013). Data-driven nonparametric prediction intervals. \emph{Journal of Statistical Planning and Inference}~143, 1039--1048.

\item{} Gneiting, T.~and Katzfuss, M.~(2014).  Probabilistic forecasting. \emph{Annual Review of Statistics and Its Applications}~1:125--151,

\item{} Hong, L.~and Martin, R.~(2017). A flexible Bayesian non-parametric model for predicting future insurance claims.  \emph{North American Actuarial Journal}~21(2), 228--241.

\item{} Hong, L., Kuffner, T. ~and Martin, R. ~(2018a). On overfitting and post-selection uncertainty assessments.  \emph{Biometrika}~105(1), 221--224.

\item{} Hong, L., Kuffner, T. ~and Martin R. ~(2018b). On prediction of future insurance claims when the model is uncertain. \emph{Variance}~12(1), 90--99.

\item{} Hong, L. and Martin, R. (2020). Model misspecification, Bayesian versus credibility estimation, and Gibbs posteriors. \emph{Scandinavian Actuarial Journal}~2020(7), 634--649.

\item{} Hong, L.~and Martin, R.~(2021). Valid model-free prediction of future insurance claims. \emph{North American Actuarial Journal}~25(4), 473--483.

\item{} Hong, L.~(2023). Conformal prediction credibility intervals. \emph{North American Actuarial Journal}~27(4), 675--688.

\item Jeon, Y.~and Kim, J.~H.~T.~(2013). A gamma kernel density estimation for insurance loss data. \emph{Insurance: Mathematics and Economics}~53, 569--579.

\item{} Kuchibhotla, A.K. Kolassa, J.E.~and Kuffner, T.A.~(2022). Post-Selection Inference. \emph{Annual Review of Statistics and Its Application}~9,  505--527.

\item{} Lei, J., Robins, J.~and Wasserman, L.~(2013). Distribution-free prediction sets. \emph{Journal of American Statistical Association}~108(501), 278--287.

\item{} Lei, J.~and Wasserman, L.~(2014). Distribution-free prediction bands for non-parametric regression.  \emph{Journal of Royal Statistical Society}-Series~B~(76), 71--96.

\item{} Olive, D.J., Rathnayake, R.C.~and Haile, M.G.~(2022). Prediction intervals for GLMs, GAMs, and some survival regression models. \emph{Communications in Statistics---Theory and Methods}~51(22), 8012--8026.

\item{} Richardson, R.~and Hartman, B.~(2018). Bayesian non-parametric regression models for modeling and predicting healthcare claims. \emph{Insurance: Mathematics and Economics}~83, 1--8.

\item{} Schervish, M.J.~(1995). \emph{Theory of Statistics}. Springer: New York.

\item{} Shafer, G.~and Vovk, V.~(2008). A tutorial on conformal prediction. \emph{Journal of Machine Learning}~9, 371--421.

\item{} Shmueli, G.~(2010). To explain or to predict? \emph{Statistical Sciences}~25(3), 289--310. 

\item{} Solvency II~(2009). \url{http://eur-lex.europa.eu/LexUriServ/LexUriServ.do?uri=OJ:L:2009:335:0001:0155:en:PDF}. Accessed on December 1, 2024.

\item{} Trevor,  H., Tibshirani, R.~and  Hastie, T.~(2009).  \emph{Elements of Statistical Learning}, Second Edition. Springer: New York.

\item{} Tukey, J.W.~(1962). The future of data analysis. \emph{Annals of Mathematical Statistics}~33, 1--67. 

\item{}  van der Vaart, A.~W.~(1998). \emph{Asymptotic Statistics}, New York: Cambridge University Press.

\item{} Vovk, V., Gammerman, A., and Shafer, G.~(2005). \emph{Algorithmic Learning in a Random World}. New York: Springer.

\item{} Vovk, V., Shen, J., Manokhin, V., Xie, M.~(2019). Non-parametric predictive distributions based on conformal prediction. \emph{Machine Learning}~108: 445--474.

\item{} Vovk, V.~(2012). Conditionally validity of inductive conformal features. \emph{Journal of Machine Learning Research: Workshop and Conference Proceedings}~25:475--490.

\item{} Vovk, V.~(2024). Conformal predictive distribution: an approach to non-parametric fiducial prediction. \emph{Chapter 17, Handbook of Bayesian, Fiducial, and Frequentst Inference}, Edited by Berger, J., Meng, X., Reid, N., and Xie, M.. CRC Press: Boca Raton.

\item{} Wilks, W.~(1941). Determination of sample sizes for setting tolerance limits. \emph{Annals of Mathematical Statistics}~12, 91--96.

\end{description}

\end{document}